\documentclass{tlp}

\newcommand\bcmdtab{\noindent\bgroup\tabcolsep=0pt
  \begin{tabular}{@{}p{10pc}@{}p{20pc}@{}}}
\newcommand\ecmdtab{\end{tabular}\egroup}

\usepackage{times}

\usepackage{color}
\usepackage{xspace}
\usepackage{graphicx} 
\usepackage{amssymb} 
\usepackage{multirow,caption,hhline} 
\usepackage{balance}
\usepackage[noend]{algorithmic}
\newtheorem{definition}{Definition}
\newtheorem{example}{Example}
\newtheorem{proposition}{Proposition}
\newtheorem{thm}{Theorem}

\usepackage{mathtools}

\newcommand{\supportplus}{\ensuremath{\overset{+}{\Rightarrow}}}
\newcommand{\supportone}{\ensuremath{\overset{\beta_1}{\Rightarrow}}}
\newcommand{\supporttwo}{\ensuremath{\overset{\beta_2}{\Rightarrow}}}
\newcommand{\supportn}{\ensuremath{\overset{\beta_n}{\Rightarrow}}}

\newcommand{\wrt}{w.r.t.\xspace}
\renewcommand{\iff}{iff\xspace}
\newcommand{\ie}{\textit{i.e.}\xspace}

\newcommand{\Set}[1]{\ensuremath{\mathbf{#1}}}

\newcommand{\AFargsNew}{\ensuremath{A}\xspace}
\newcommand{\AFattRelNew}{\ensuremath{\Omega}\xspace}
\newcommand{\AFNew}{\ensuremath{\langle \AFargsNew, \AFattRelNew\rangle}\xspace}

\def\imply{{\text{implies}}}

\def\aaf{{\mathcal{A}}}

\def\S{\ensuremath{\sigma}}
\def\ssup{\ensuremath{\textit{sup}}}

\def\reach0{\it Reach\mbox{$\!\!_{{_\aaf}_0}\!$}}

\def\gr{ \mbox{$\tt gr$}}
\def\pr{ \mbox{$\tt pr$}}
\def\co{ \mbox{$\tt co$}}
\def\sta{ \mbox{$\tt st$}}

\usepackage{soul}

\def\+{\mbox{+}}
\def\-{\mbox{-}}
\def\<{\langle}
\def\>{\rangle}

\newcommand{\ldot}{\,{\bf .}\,}
\def\nrDef{\ensuremath{\textit{def}}}
\def\nrAcc{\ensuremath{\textit{acc}}}

\newcommand{\Def}{\textsc{Def}}
\newcommand{\Acc}{\textsc{Acc}}

\def\={\mbox{=}}

\def\true{\mbox{$\tt true$}}

\def\a{\mbox{$\tt a$}}
\def\b{\mbox{$\tt b$}}
\def\c{\mbox{$\tt c$}}

\def\Sb{\Set{S}}
\def\sb{\Set{s}}
\def\tb{\Set{t}}

\newcommand{\blue}[1]{{\color{blue}{#1}}}

\begin{document}

\lefttitle{G. Alfano et al.}

\jnlPage{1}{18}
\jnlDoiYr{2024}
\doival{10.1017/xxxxx}

\title[Cyclic Supports in Recursive Bipolar AFs: Semantics and LP Mapping]{Cyclic Supports in Recursive Bipolar Argumentation Frameworks: Semantics and LP Mapping\thanks{We acknowledge the support from project Tech4You (ECS0000009), and PNRR MUR projects FAIR (PE0000013) and SERICS (PE00000014).}}

\begin{authgrp}
\author{\sn{Gianvincenzo} \gn{Alfano}, \sn{Sergio} \gn{Greco}\\ \sn{Francesco} \gn{Parisi},   and \sn{Irina} \gn{Trubitsyna}}
\affiliation{DIMES Department, University of Calabria, Rende, Italy\\
(e-mail: $\{$g.alfano,\ greco,\ fparisi,\ i.trubitsyna$\}$@dimes.unical.it)}
\end{authgrp}

\history{\sub{xx xx xxxx;} \rev{xx xx xxxx;} \acc{xx xx xxxx}}

\maketitle

\begin{abstract}
Dung's Abstract Argumentation Framework (AF) has emerged as a key formalism for argumentation in Artificial Intelligence.  It has been extended in several directions, including the possibility to express supports, leading to the development of the Bipolar Argumentation Framework (BAF), and recursive attacks and supports, resulting in the Recursive BAF (Rec-BAF).  Different interpretations of supports have been proposed, whereas for Rec-BAF (where the target of  attacks and supports may also be attacks and supports) even different semantics for attacks have been defined. However, the semantics of these frameworks have either not been defined in the presence of support cycles, or are often quite intricate in terms of the involved definitions. 
We encompass this limitation and present classical semantics for general BAF and Rec-BAF and show that the semantics for specific BAF and Rec-BAF frameworks can be defined by very simple and intuitive modifications of that defined for the case of AF. 
This is achieved by providing a modular definition of the sets of defeated and acceptable elements for each AF-based framework. We also characterize, in an elegant and uniform way, the semantics of general  BAF and Rec-BAF in terms of logic programming and partial stable model semantics.
\end{abstract}

\begin{keywords}
Abstract argumentation,  argumentation semantics, partial stable models
  \end{keywords}

\section{Introduction}\label{sec:Intro}

Formal argumentation has emerged as one of the important fields in Artificial Intelligence~\cite{Rahwan-Simari09}.
In particular, Dung's abstract Argumentation Framework (AF) is a simple, yet powerful formalism for modelling disputes between two or more agents~\cite{Dung95}.
An AF consists of a set of \emph{arguments} and a binary \emph{attack} relation over the set of arguments that specifies the \textit{interactions} between arguments:
intuitively, if argument $a$ attacks argument $b$, then $b$ is acceptable only if $a$ is not.
Hence, arguments are abstract entities whose role is entirely determined by the interactions specified by the attack relation.

Dung's framework has been extended in many different ways, including the introduction of new kinds of interactions between arguments and/or attacks.
In particular, the class of Bipolar Argumentation Frameworks (BAFs) is an interesting extension of AF which allows for also modelling \textit{supports} between arguments~\cite{Nouioua-Risch11,Villata-Boella-Gabbay-vanDerTorre12}.
Different interpretations of supports have been proposed in the literature. 
Further extensions consider second-order interactions~\cite{Villata-Boella-Gabbay-vanDerTorre12}, e.g., attacks to attacks/supports, as well as more general forms of interactions such as recursive AFs where attacks can be recursively attacked~\cite{BaroniJAR11,CayrolFCL17} and  recursive BAFs, where attacks/supports can be recursively attacked/supported~\cite{Gottifredi-Cohen-Garcia-Simari18}. 
The next example shows two extensions of AF, whereas an overview of the extensions of AF studied in this paper is provided at the end of this section.

\begin{example}\label{ex1:intro}\rm 
Consider a scenario regarding a restaurant and the possible menus to be suggested to customers.
The arguments denote elements which are available and that can occur in a menu, whereas, attacks denote incompatibility among elements.
Consider the scenario shown in Figure \ref{fig:intro} (left), where 
the attacks $\alpha_i$ ($1 \leq i \leq 4$) denote incompatibilities, whereas the support $\beta_1$ from $\tt fish$ to $\tt white$ states that a menu may have white wine only if it also contains fish.
Regarding the attacks, they state that the menu cannot contain both meat and fish as well as both white and red wine.

Assume now there exist two arguments $\tt sorbet$ and $\tt s\overline{orbe}\tt t$ (no sorbet), attacking each-other, stating that the menu may contain or not contain the sorbet.
The resulting framework is shown in Figure \ref{fig:intro} (right). 
The (recursive) attacks $\alpha_7$ and $\alpha_8$ from $\tt sorbet$ to $\alpha_1$ and $\alpha_2$ state that if sorbet is in the menu, then the incompatibility between fish and meat is not valid anymore.~\hfill~$\Box$
\end{example}

\begin{figure}[t!]\centering
\centering
\includegraphics[scale=0.65]{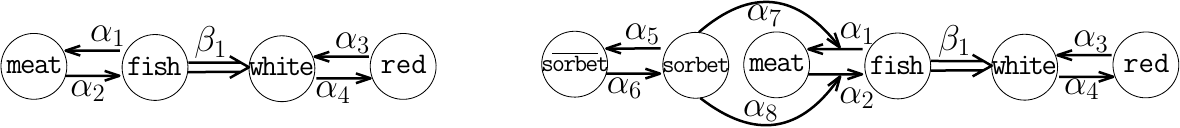}
\caption{\small BAF (left) and Rec-BAF (right) of Example~\ref{ex1:intro}}
\label{fig:intro}
\end{figure}

While the semantics of argumentation is universally accepted, for the frameworks extending AF with supports and recursive attacks/supports, several interpretations of their role have been proposed, giving rise to different semantics \cite{Rahwan-Simari09,CayrolCLS21,CohenGGS14}.
Following Dung's approach, the meaning of recursive AF-based frameworks is still given by relying on the concept of extension.
However, the extensions of an \textit{AF with recursive attacks and supports} also include the (names of) attacks and supports, that intuitively contribute to determine the set of accepted arguments \cite{CayrolFCL17,CayrolFCL18}.

\vspace*{2mm}
\noindent
\textbf{AF-based frameworks.}
Among the different frameworks extending AF some of them share the same structure, although they have different semantics.
Thus, in the following we distinguish between framework and \textit{class} of frameworks.
Two frameworks sharing the same syntax (i.e. the structure) belong to the same (syntactic) class. 
For instance, BAF is a syntactic class, whereas AFN and AFD are two specific frameworks sharing the same BAF syntax; their semantics differ because they interpret supports in different ways.

Figure \ref{fig:overview} overviews the frameworks extending AF studied in this paper. 
Horizontal arrows denote the addition of supports with two different semantics (necessary semantics in the left direction and deductive semantics in the right direction), whereas vertical arrows denote the extension with recursive interactions (i.e., attacks and supports); the two directions denote two different semantics proposed in the literature {for determining the acceptance status of attacks}. 
The frameworks ASAF and RAFN (as well as AFRAD and RAFD), belong to the same class, called \emph{Recursive BAF} (Rec-BAF), combining AFN (resp. AFD) with AFRA and RAF, respectively.
The differences between ASAF and RAFN (resp. AFRAD and RAFD) semantics are not in the way they interpret supports, both based on the necessity (resp. deductive) interpretation, but in a different determination of the status of attacks as they extend AFRA and RAF, respectively.
Clearly, frameworks in the corners are the most general ones. However, for the sake of presentation, before considering the most general frameworks, we also analyze the case of BAFs.

\begin{figure}\label{fig:overview}
\hspace*{-4mm}
\raisebox{3mm}{
    \begin{minipage}[b]{0.35\linewidth}
        \centering
        \includegraphics[width=1.20\textwidth]{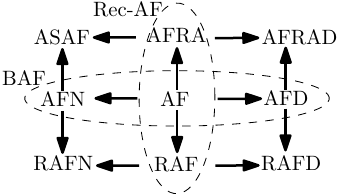}
    \end{minipage}}
    \hspace*{+1mm}
    \raisebox{3mm}{
    \begin{minipage}[b]{0.64\linewidth}
    \footnotesize
        \centering
        Legend of Acronyms:
        \vspace*{-1mm}
        \begin{itemize}
        \itemsep=0pt
        \item[] ~{\bf AF:} abstract Argumentation Framework [Dung,1995]
        \item[] ~{\bf BAF:} Bipolar AF
        		\item[] ~~~{\bf AFN:} AF with Necessities [Nouia and Risch,2011]
        		\item[] ~~~{\bf AFD:} AF with {D}eductive supports [Villata et al.,2012]
        	\item[] ~{\bf Rec-AF:} Recursive-AF
        		\item[] ~~~{\bf AFRA:} AF with Recursive Attacks [Baroni et al.,2011]
        		\item[] ~~~{\bf RAF:} Recursive AF [Cayrol et al.,2017]
        	\item[] ~{\bf Rec-BAF:} Recursive-BAF 
        		\item[] ~~~{\bf ASAF:} Attack-Support AF [Gottifredi et al.,2018]
        		\item[] ~~~{\bf RAFN:} Recursive AF with Necessities [Cayrol et al.,2018]
        		\item[] ~~~{\bf AFRAD:} AF with Rec. Att. and Ded.  supports [Alfano et al.,2020]
        		\item[] ~~~{\bf RAFD:} Recursive AF with Deductive supports [Alfano et al.,2020]
        \end{itemize}
    \end{minipage}}
    \hspace*{-5mm}
    \vspace*{-5mm}
    \centering
    \caption{AF-based frameworks investigated in the paper. Rec-BAF frameworks are in the corners.}\label{fig:overview}
    \vspace*{-5mm}
\end{figure}

So far, the semantics of BAF and Rec-BAF frameworks has been defined only for restricted classes, called \textit{acyclic}.
Although recently two new semantics have been proposed for general AFN and RAFN, these semantics extending the ones defined for acyclic frameworks are quite involved.
Thus, in this paper we propose simple extensions of the semantics defined for acyclic BAF and Rec-BAF that apply to all specific frameworks in these classes (e.g. AFN, AFD, ASAF, RAFN, etc.) and show that, as for the acyclic case, the semantics of a framework can be also computed by considering an `equivalent' logic program under partial stable model semantics.

\vspace*{2mm}
\noindent
\textbf{Contributions.}
The main contributions are as follows:

\vspace*{-2mm}
\begin{itemize}
\item 
We propose classical semantics for general BAF and Rec-BAF and show that the semantics for specific BAF and Rec-BAF frameworks can be defined by very simple and intuitive modifications of that defined for the case of AF;
\item
We show that any general (Rec-)BAF $\Delta$ can be mapped into an equivalent logic program $P_\Delta$ under (partial) stable model semantics;
\item
We investigate the complexity of the verification and credulous and skeptical acceptance problems for (Rec-)BAF under semantics $\S\in\{\tt gr, co, st, pr\}$.  It turns out that BAF and Rec-BAF are as expressive as AF,  though more general relations (i.e., supports and recursive relations) can be easily expressed.
\end{itemize}

\vspace*{-2mm}
\section{Preliminaries}\label{sec:Preliminaries}

\vspace*{-2mm}
\subsection{Argumentation Frameworks}\label{sec:AF}

\vspace*{-1mm}
An abstract \textit{Argumentation Framework} (AF) is a pair \AFNew, 
where $\AFargsNew$ is a set of \textit{arguments} and $\AFattRelNew \subseteq \AFargsNew \times \AFargsNew$ is a set of \textit{attacks}.
An AF can be seen as a directed graph, whose nodes represent arguments and edges represent attacks;  an attack $(a,b) \in \Omega$ from $a$ to $b$ is represented by $a \rightarrow b$.

Given an AF $\Delta=$\AFNew  and a set $\Sb \subseteq \AFargsNew$ of arguments, 
an argument $a \in \AFargsNew$ is said to be
\textit{i})	
\emph{defeated} w.r.t. $\Sb$ iff $\exists b \in \Sb$  such that 
$(b, a) \in  \Omega$, and 
\textit{ii})
\emph{acceptable} w.r.t. $\Sb$ iff for every argument 
$b \in  \AFargsNew$ with $(b, a) \in \Omega$, there is
$c \in \Sb$ such that $(c,b) \in \Omega$.
The sets of defeated and acceptable arguments w.r.t. $\Sb$ (where $\Delta$ is understood) are next defined.

\begin{definition}
For any AF $\Delta = \<A,\Omega\>$ and set of arguments $\Sb \subseteq A$, the set of arguments defeated by $\Sb$ and acceptable \wrt $\Sb$ are defined as follows:\\
\noindent
$\bullet$ $\nrDef(\Sb) = \{ a \in A\ |\ \exists\ b \in \Sb \ldot b\rightarrow a \}$;\\
$\bullet$ $\nrAcc(\Sb)\ = \{ a \in A\ |\ \forall\ b \in A \ldot b\rightarrow a\ \imply\ b \in \nrDef(\Sb) \}$.
\end{definition}

Given an AF \AFNew, a set $\Sb \subseteq \AFargsNew$ of arguments is said to be $i)$
\emph{conflict-free} \iff $\Sb \cap \nrDef(\Sb) = \emptyset$, and $ii)$ \emph{admissible} \iff it is conflict-free and $\Sb \subseteq \nrAcc(\Sb)$.
Moreover,  $\Sb$ 
 is a: \\
$\bullet$ \emph{complete extension} \iff it is conflict-free and $\Set{S} = \nrAcc(\Set{S})$; \\
$\bullet$ \emph{preferred extension} iff it is a $\subseteq$-maximal complete extension;\\
$\bullet$ \emph{stable extension} \iff it is a total (i.e. $\Set{S} \cup \nrDef(\Set{S}) = A$) preferred extension;\\
$\bullet$ \emph{grounded extension} \iff it is the $\subseteq$-smallest complete extension.

The set of complete (resp., preferred, stable, grounded) extensions of a framework $\Delta$ will be denoted by ${\tt co}(\Delta)$ (resp., ${\tt pr}(\Delta)$, ${\tt st}(\Delta)$, ${\tt gr}(\Delta)$).
We often denote extensions by pairs denoting both accepted and defeated elements using the notation $\widehat{\sigma}(\Delta)$ instead of $\sigma(\Delta)$ (with $\sigma \in\tt \{\co,\sta,\pr,\gr\}$), where $\widehat{\sigma}(\Delta) = \{ \widehat{\Sb}=\<\Sb,\nrDef(\Sb)\> \mid \Sb \in \sigma(\Delta) \}$.

\begin{example}\label{ex:prel-af}\rm
Let $\Delta=\< A,\Omega\>$ be an AF where $A=\tt \{ a, b, c, d \}$ and $\Omega=\{\tt (a,b), (b,a), (a,c), (b,c), (c,d),$ $\tt (d,c) \}$. The set of complete extensions of $\Delta$ is  $\co(\Delta)=\tt  \{ \emptyset, \{ d \}, \{ a, d \}, \{ b, d \} \}$. \ 
Consequently,  \ 
$\pr(\Delta) = \sta(\Delta) =\tt  \{ \{ a, d \}, \{ b, d \} \}$, $\gr(\Delta)=\tt  \{ \emptyset \}$.~\hfill~$\square$
\end{example}

Two main problems in formal argumentation are \textit{verification} and  \textit{acceptance}~\cite{DvorakGRW23,FazzingaFF22,AlfanoGPT22ijcai,ALFANO2023103967,AlfanoGPT23aamas,ALFANO2023757,AlfanoGPT23ijcai,ALFANO2024aij-weak,AlfanoGPT24aaai,DvorakH0SSW22,DvorakKUW24}.
Given an AF $\Delta=\<A,\Omega\>$, a set $S\subseteq A$ of arguments,  and a semantics $\sigma\in\{\gr,\co,\sta,\pr\}$, the \textit{verification} problem 
is the problem of deciding  $S$ is a $\S$-extension of $\Delta$ (i.e., $S\in\S(\Delta)$).
Given an AF $\Delta=\<A,\Omega\>$, an argument $a\in A$, and a semantics $\sigma\in\{\gr,\co,\sta,\pr\}$, the \textit{credulous (\mbox{resp.} skeptical) acceptance} problem
is the problem of deciding whether argument $a$ is credulously (resp. skeptically) accepted {under semantics \S}, that is, deciding whether $a$ belongs to at least one (resp. all) $\sigma$-extension of the framework $\Delta$. 
Clearly, for the grounded semantics, which admits exactly one extension, these problems become identical.
The complexity of acceptance and  verification problems for AF
has been thoroughly investigated (see~\cite{DvorakD17} for a survey).
The semantics of an AF $\Delta$ can be also computed by considering the logic program $P_\Delta$, under partial stable model (PSM) semantics, derived from $\Delta = \<A,\Omega\>$ as follows:
\begin{center}
\hspace{3cm}
$P_{\Delta} = \left\{ a \leftarrow \bigwedge_{(b,a) \in \Omega} \neg b \mid a \in A \right\}_{.}$ 
\hspace{3cm} $(1)$
\end{center}

Partial stable models are 3-valued. 
A partial interpretation $M$ is a partial stable model of a ground program $P$ if it is the minimal model of the program $P^M$ obtained from replacing negated literals in rules' body with their truth values in $M$. 
The set of partial stable models of a program $P$ is denoted as ${\cal PS}(P)$.
It has been shown that for any AF $\Delta$, $\widehat{\co}(\Delta) = {\cal PS}(P_\Delta)$.
More information on partial stable model semantics can be found in Appendix A.

\vspace*{-3mm}
\subsection{Bipolar Argumentation Frameworks}

\vspace*{-1mm}
A \textit{Bipolar Argumentation Framework} (BAF) is a triple $\< A, \Omega, \Gamma\>$, where $A$ is a set of \textit{arguments}, $\Omega \subseteq A \times A$ is a set of \textit{attacks}, and $\Gamma \subseteq A \times A$ is a set of \textit{supports}.
A BAF can be  represented by a directed graph with two types of edges: \emph{attacks} and \emph{supports}, denoted by $\rightarrow$ and $\Rightarrow$, respectively. 
A \emph{support path} $a_0 \supportplus a_n$ from argument $a_0$ to argument $a_n$ is a  sequence of $n$ edges $a_{i-1} \Rightarrow a_i$ with $0 < i \leq n$.
We use $\Gamma^+ =$ $\{ (a,b) \ |$ $a,b \in A \ \wedge a \supportplus b \}$ to denote the set of pairs $(a,b)$ such that there exists a support path from $a$ to $b$.
A BAF is said to be acyclic if it does not have support cycles, i.e. there is no argument $a$ such that $a \supportplus a$.

Different interpretations of the support relation have been proposed in the literature \cite{Rahwan-Simari09,CayrolCLS21}. 
In this paper we concentrate on the necessity and deductive interpretations \cite{Nouioua-Risch11,Villata-Boella-Gabbay-vanDerTorre12}.
Necessary support is intended to capture the following intuition: if argument $a$ supports argument $b$ then the acceptance of $b$ implies the acceptance of $a$ and the non-acceptance of $a$ implies the non-acceptance of $b$, whereas the deductive interpretation states that the acceptance of $a$ implies that $b$ is also accepted.

The two BAF frameworks obtained by interpreting the supports as necessities or as deductive are respectively called  Argumentation Framework with Necessities (AFN) and Argumentation Framework with Deductive Supports (AFD).

It is worth noting that for AFN the following implication hold: 
$(i)$ $a \rightarrow b$ and $b \supportplus c$ imply $a \rightarrow c$ (called \emph{secondary or derived attack}). 
Similarly, for AFD the following implication hold: 
$(ii)$ $a \rightarrow b$ and $c \supportplus b$ imply $a \rightarrow c$  (called \emph{mediated or derived attack}). 
Considering Example~\ref{ex1:intro}, under the necessary interpretation of supports, there exist a secondary attack from $\tt meat$ to $\tt white$.
	
There has been a long debate on the fact that the support relation should be acyclic as it is inherently transitive.
On the other side, a framework represents a situation where several agents may act and, thus, support cycles could be present.
In this subsection we assume that the support relation is acyclic, as originally defined in \cite{Nouioua-Risch11}. 
A semantics for the general AFN, albeit a bit involved, has been recently proposed in \cite{NouiouaB23}.

The aim of this paper, explored in Section~\ref{sec:RBAF}, is the definition of intuitive semantics for general Recursive BAF, that  extend both semantics defined for BAF and acyclic BAF.
To this end, in the rest of this subsection we discuss semantics for acyclic BAF, whereas
in the Section~\ref{sec:BAF} we present an intuitive semantics for general BAF.
For BAF with supports interpreted as necessities (AFN) the new semantics is equivalent to that proposed in \cite{NouiouaB23}.

\vspace*{1mm}
\noindent
{\bf AF with Necessary Supports (AFN).}
The semantics of AFN with acyclic supports can be defined by first redefining the definition of defeated and acceptable sets as follows.

\vspace*{-1mm}
\begin{definition}\label{def:AFN}
For any AFN $\<A, \Omega, \Gamma\>$ and set of arguments $\Sb \subseteq A$: \\
$\bullet$ 
$\nrDef(\Sb) = \{ a \in A\ |\ (\exists b \in \Sb \ldot b\rightarrow a)\ \vee$
$(\exists c \in \nrDef(\Sb) \ldot  c \Rightarrow a  \}$; \\
$\bullet$
$\nrAcc(\Sb) = \{ a \in A\ |\ (\forall b\! \in\! A \ldot  b\rightarrow a\ \imply\ b\! \in\! \nrDef(\Sb)) \wedge (\forall c\in A \ldot  c\Rightarrow a\ \imply\  c\! \in\! \nrAcc(\Sb) \}$.
\end{definition}

It is worth noting that $\nrDef(\Sb)$ and $\nrAcc(\Sb)$ are defined recursively.
The definitions of conflict-free and admissible sets, as well as the definitions of complete, preferred, stable and grounded extensions are the same of those introduced for AF.

\vspace*{-0mm}
\begin{example}\label{ex:BAF2}\rm
Let $\Delta=\tt \< \{ \tt fish, meat, white, red\},$ 
$\tt \{\tt (fish,meat),$ $\tt (meat,fish),$ $\tt  (white,red),$ $\tt (red,white)\},$ $\{\tt (fish,white)\} \>$ be  the  AFN shown on Figure~\ref{fig:intro}(left). 
Then, $\nrDef(\{\tt fish,white\})$ $= \{ \tt meat,red \}$ and $\nrAcc(\{\tt fish,white\}) =$ $ \{ \tt fish,white \}$. 
$\Delta$  has six complete extensions, that are:  
$E_0 = \emptyset$,  
$E_1=\{\tt red\}$, 
$E_2=\{\tt fish\}$, 
$E_3=\{\tt fish, red\}$, 
$E_4=\{\tt fish, white\}$, 
$E_5=\{\tt meat,red\}$.
Moreover, $E_3$, $E_4$, and $E_5$ are preferred (and also stable), while $E_0$ is the grounded extension.~\hfill~$\square$
\end{example}

\vspace*{-0mm}
An alternative and equivalent way to define the semantics of acyclic AFN is to rewrite them in terms of `equivalent' AF, by introducing secondary attacks and deleting supports. 
The derived AF of the AFN 
$\<A, \Omega, \Gamma\>$ of Example~\ref{ex:BAF2} is 
$\<A,\Omega\cup\{\tt (meat,white)\}\>$.

\vspace*{2mm}
\noindent
{\bf AF with Deductive Supports (AFD).}
The semantics of AFD  is dual w.r.t. that of AFN, that is we can transform a BAF with deductive supports into an equivalent BAF with necessary supports by simply reversing the direction of the support arrows~\cite{CayrolCLS21}.  Equivalently,  the semantics of acyclic AFD can be defined in terms of AF by adding {mediated attacks} and removing the supports.  Alternatively,  it can be presented as in Definition~\ref{def:AFN} by replacing the support $c\Rightarrow a$ with the support $a\Rightarrow c$.

\subsection{Recursive BAF}\label{sec:ASAFpreliminary}

\vspace*{-1mm}
By combining the concepts of bipolarity and recursive interactions, more general argumentation frameworks have been defined.
A \textit{Recursive Bipolar Argumentation Framework (Rec-BAF)} is a tuple $\< A, R, T, {\bf s}, {\bf t} \>$, where $A$ is a set of arguments, $R$ is a set of attack names, $T$ is a set of support names, ${\bf s}$ (resp., ${\bf t}$) is a function from $R\ \cup\ T$ to $A$ (resp., to $A\ \cup\ R\ \cup\ T$), that is mapping each attack/support to its source (resp., target). 

Considering AF with recursive attacks, i.e. AF in which the target of an attack can also be an attack, two different semantics have been defined, giving rise to two specific frameworks:
Recursive Abstract Argumentation Framework (RAF) \cite{CayrolFCL17} and Abstract Argumentation Framework with Recursive Attacks (AFRA) \cite{BaroniJAR11}.
We do no further discuss these two frameworks as they are just special cases of Rec-BAF. 
We mention them only for the fact that they were proposed before  Rec-BAF frameworks, and the differences in their semantics, combined with different interpretation of supports, give rise to different specific Rec-BAF frameworks, namely the ones appearing at the corner of Figure \ref{fig:overview}.
We first discuss the two frameworks where supports are interpreted as necessities:
Recursive Argumentation Framework with Necessities (RAFN) (extending RAF) \cite{CayrolFCL18}, and Attack Support Argumentation Framework (ASAF) (extending AFRA) \cite{Gottifredi-Cohen-Garcia-Simari18}.
Regarding the extensions of RAF and AFRA with deductive supports, called  Recursive Argumentation Framework with Deductive Supports (RAFD), and  Argumentation Framework with Recursive Attacks-Supports (AFRAD) \cite{Alfano-TPLP}, their semantics will be recalled at the end of this section. 

We now discuss semantics for  `acyclic' Rec-BAF, although recently there have been two contributions defining the semantic of general RAF and RAFN~\cite{LagasquieSchiex23}.
We do not further discuss these semantics as they are quite involved and in the next two sections we present our main contribution, consisting in the definition of new semantics for general BAF and Rec-BAF, extending in a natural way the semantics defined for acyclic BAF and acyclic Rec-BAF.
As the underlying structure representing a Rec-BAF is not a graph, the definition of (a)cyclicity has been formulated in terms of acyclicity of the BAF obtained by replacing every attack $a \rightarrow^{\!\!\!\!\!\alpha} b$ with $a \Rightarrow \alpha$ and $\alpha \rightarrow b$, and  every support $a \Rightarrow^{\!\!\!\!\!\!\beta} b$ with $a \Rightarrow {\beta}$ and ${\beta} \Rightarrow b$.
Note that the so-obtained auxiliary BAF is only used to formally define and check acyclicity in Rec-BAF, not to provide the semantics which are instead recalled next.

\vspace*{1mm}
\noindent 
{\bf Recursive AF with Necessities (RAFN).} 
The RAFN framework 
has been proposed in~\cite{CayrolFCL18}. 
The semantics combines the RAF interpretation of attacks in RAF with the necessity interpretation of supports of AFN. 
Here we consider a simplified version where supports have a single source and the support relation is acyclic.
Differences between RAFN and ASAF semantics are highlighted in blue.

\begin{definition}\label{def:RAFN-sem}
For any acyclic RAFN $\< A, \Sigma, T, {\bf s}, {\bf t} \>$ and set $\Sb \subseteq A \cup R \cup T$, we have that: 

\vspace*{1mm}
\noindent
$\bullet\ {\nrDef}(\Sb) = \{ X \in A \cup R \cup T\ |\ 
(\exists \alpha \in R \cap \Sb \ldot  \tb(\alpha)=X \wedge \blue{\sb(\alpha) \in \Sb}) \ \vee \\ 
\hspace*{45mm}\! (\exists~\beta~\in~T \cap~\Sb \ldot \tb(\beta)=X \wedge {\sb~(\beta) \in \nrDef(\Sb)})\ \}$;

\noindent
$\bullet\ \nrAcc(\Sb)\! =\!\{ X\! \in A\cup R\cup T\ |\ 
(\forall \alpha\! \in\! R  \ldot \tb(\alpha)=X\ \imply\ (\alpha \in \nrDef(\Sb) \vee \blue{{\bf s}(\alpha) \in\! \nrDef(\Sb)}))\wedge \\
\hspace*{47mm} (\forall \beta\! \in\! T \ldot \tb(\beta)=X\!\ \imply\     (\beta \in \nrDef(\Sb) \vee {\bf s}(\beta) \in \nrAcc(\Sb)))\ \}$.
\end{definition}

\vspace*{2mm}
\noindent
{\bf Attack-Support AF (ASAF).}
The semantics combines the AFRA interpretation of attacks with that of BAF under the necessary interpretation of supports (i.e., AFN). For the sake of presentation, we refer  to the formulation presented in \cite{ecai2020,AlfanoCGGPS24}, where attack and support names are first-class citizens, giving the possibility to represent multiple attacks and supports from the same source to the same target. 
For any ASAF $\Delta$ and $\Sb\subseteq A \cup R \cup T$, the \emph{defeated} and \emph{acceptable} sets (given $\Sb$) are defined as follows.

\begin{definition}\label{def:ASAF-sem}
Given an acyclic ASAF $\< A, R, T, {\bf s}, {\bf t} \>$ and a set $\Sb \subseteq A \cup R \cup T$, we define: 

\vspace*{1mm}
\noindent
$\bullet\ {\nrDef}(\Sb) = \{ X \in A \cup R \cup T\ |\ \blue{(X \in R \wedge \sb(X) \in \nrDef(\Sb))\  \vee}\\
\hspace*{44mm} {(\exists \alpha \in R \cap \Sb \ldot 
	\tb(\alpha)=X)}\ \vee \\  
\hspace*{44mm} (\exists \beta \in T \cap \Sb \ldot 
	\tb(\beta)=X\ \wedge\ \sb(\beta) \in \nrDef(\Sb))\}$; 
	
\vspace*{1mm}
\noindent
$\bullet\ \nrAcc(\Sb) = \{ X \!\in\! A\cup R\cup T\ |\ \blue{(X\in R\ \imply\ \sb(X) \in \nrAcc(\Sb))}\wedge\\
\hspace*{44mm} (\forall \alpha \in R\!\ldot{\bf t}(\alpha)=X\ \imply\ \alpha\in\nrDef(\Sb)) \wedge \\
\hspace*{44mm} (\forall \beta\in T \ldot  \tb(\beta)=X\ \imply\ (\beta\in \nrDef(\Sb) \vee \sb(\beta)\in \nrAcc(\Sb)))\}$.
\end{definition}

\vspace*{1mm}
Again, the notions of \emph{conflict-free}, \emph{admissible sets}, and the different types of extensions can be defined in a standard way (see Section \ref{sec:AF}) by considering $\Sb\subseteq A \cup R \cup T$
and by using the definitions of defeated and acceptable sets reported above.
It is worth noting that, the differences between ASAF and RAFN semantics (highlighted in blue) are not in the way they interpret supports (both based on the necessity interpretation), but in a different determination of the status of attacks as they extend AFRA and RAF, respectively. Moreover, for each semantics, the RAFN extensions can be derived from the corresponding ASAF extensions and vice versa.

\begin{example}\label{ex:prel-afras}\rm
Consider the Rec-BAF $\Delta$ with necessary supports of Figure~\ref{fig:intro}(right) and assume arguments are denoted by their initials.   
The preferred (and also stable) extensions prescribing sorbet in the menu under RAFN (resp., ASAF) semantics are: 
$E_0=\{\tt s,m,f,w,$ $\beta_1,$ $\alpha_3,$ $\dots,$ $\alpha_8\}$  and 
$E_1=\{\tt s,m,f,r,$ $\beta_1,$ $\alpha_3,$ $\dots,$ $\alpha_8\},$  
(resp.,  $E_0'=\{\tt s,m,f,w,$ $\beta_1,$ $\alpha_4,$ $\alpha_6,$ $\alpha_7,$ $\alpha_8\}$  and $E_1'=\{\tt s,m,f,r,$ $\beta_1,$ $\alpha_3,$ $\alpha_6,$ $\alpha_7,$ $\alpha_8\}$).
Note that, under ASAF semantics, attacks $\tt \alpha_3$ (resp., $\tt \alpha_4$) and $\alpha_6$ are not part of the extension $E_0'$ (resp., $E_1'$) as their sources (i.e., $\tt r$,  $\tt w$,  and $\tt \bar{s}$,  respectively) are defeated.~\hfill~$\square$
\end{example}

\vspace*{0mm}
\noindent
{\bf Recursive AF with Deductive Supports (RAFD).}
For BAFs,  necessary support and deductive support are dual	
(\ie it is possible to transform a BAF with necessity into an equivalent BAF with deductive supports by simply reversing the direction of the support arrows)~\cite{CayrolCLS21}. 
However, in the case of Rec-BAF that are not BAFs, this duality no longer holds. 
This happens because the target of supports and attacks in Rec-BAF may also be other supports and attacks.
For this reason,  we next explicitly recall the semantics for Rec-BAF with deductive supports.

\begin{definition}
For any acyclic RAFD $\< A, R, T, {\bf s}, {\bf t} \>$ and set $\Sb \subseteq A \cup R \cup T$, we have that: 

\vspace*{1mm}
\noindent
$\bullet\ {\nrDef}(\Sb) = \{ X \in A \cup R \cup T\ |\
(\exists \alpha \in R \cap \Sb \ldot  \tb(\alpha)=X \wedge \blue{\sb(\alpha) \in \Sb}) \ \vee \\ 
\hspace*{45mm} (\exists~\beta~\in~T \cap~\Sb \ldot \blue{\sb}(\beta)=X \wedge {\blue{\tb}(\beta) \in \nrDef(\Sb)})\ \}$;

\vspace*{1mm}
\noindent
$\bullet\ \nrAcc(\Sb)\! =\!\{ X \in A\cup R\cup T\ |\  
(\forall \alpha\! \in\! R  \ldot \tb(\alpha)=X\ \imply\ (\alpha \in \nrDef(\Sb) \vee {{\bf s}(\alpha) \in\! \nrDef(\Sb)}))\wedge \\
\hspace*{45mm} (\forall \beta\! \in\! T \ldot \blue{\sb}(\beta)=X\!\ \imply\     (\beta \in \nrDef(\Sb) \vee \blue{\bf t}(\beta) \in \nrAcc(\Sb)))\ \}$.
\end{definition}

We have highlighted in blue the differences between the RAFN and the RAFD definitions of defeated and acceptable arguments.

\vspace*{2mm}
\noindent 
{\bf AF with Recursive Attacks and Deductive Supports (AFRAD).}

\begin{definition}
Given an acyclic AFRAD $\< A, R, T, {\bf s}, {\bf t} \>$ and a set $\Set{S} \subseteq A \cup R \cup T$, we define: 

\vspace*{1mm}
\noindent
$\bullet\ {\nrDef}(\Sb) = \{ X \in A \cup R \cup T\ |\ {(X \in R \wedge \sb(X) \in \nrDef(\Sb))\  \vee}\\
\hspace*{44mm}  {(\exists \alpha \in R \cap \Sb \ldot 
	\tb(\alpha)=X)}\ \vee\  \\  
\hspace*{44mm} (\exists \beta \in T \cap \Sb \ldot 
	\blue{\sb}(\beta)=X\ \wedge\ \blue{\tb}(\beta) \in \nrDef(\Sb))\}$; 
	
\noindent
$\bullet\ \nrAcc(\Sb)\! =\!\{\! X \!\in\! A\cup R\cup T\ |\ {(X\in R\ \imply\ \sb(X) \in \nrAcc(\Sb))}\wedge \\
\hspace*{42mm} (\forall \alpha \in R\!\ldot{\bf t}(\alpha)=X\ \imply\ \alpha\in\nrDef(\Sb)) \wedge \\
\hspace*{42mm} (\forall \beta\in T \ldot  \blue{\sb}(\beta)=X\ \imply\ (\beta\in \nrDef(\Sb) \vee \blue{\tb}(\beta)\in \nrAcc(\Sb)))\}$.
\end{definition}

Again, we have highlighted in blue the differences between the ASAF and the AFRAD definitions of defeated and acceptable arguments.
Analogously to the case of RAFN vs ASAF, RAFD and AFRAD semantics may differ only in the status of attacks.

\paragraph{Mappings to other formalisms.}

It has been shown that any acyclic Rec-BAF $\Delta$ can be mapped to 
i) an `equivalent' AF $\Lambda$ so that $\co(\Delta) \equiv \co(\Lambda)$ (i.e.   $\co(\Delta) = \co(\Lambda)$ modulo meta-arguments introduced in the rewriting), and 
ii) a logic program $P_\Delta$ so that $\widehat{\co}(\Delta) = {\cal PS}(P_\Delta)$~\cite{Alfano-TPLP}.
We next extend these results to general Rec-BAF.

\def\AA{\mathbb{A}}
\def\RR{\mathbb{R}}
\def\TT{\mathbb{T}}

\section{A new semantics for BAF}\label{sec:BAF}

In this section we present a new and intuitive semantics for general BAF.
In the rest of this section, whenever we refer to a BAF we intend either an AFN (i.e. a BAF where supports are intended as necessities) or an AFD (where supports are intended as deductive).

We start by introducing new definitions for defeated and acceptable sets that extend the ones recalled in the previous section.
To distinguish the defeated and acceptable sets defined for general BAF
from those defined 
for acyclic BAF, 
we introduce new functions $\Def$ and $\Acc$.

\begin{definition}\label{def:BAF-Sem}
\noindent
For any general AFN $\<A, \Omega, \Gamma\>$ and set of arguments $\Sb \subseteq A$: \\
$\bullet$ $\Def(\Sb) = \{ a \in A\ |\ (\exists b \in \Sb \ldot b\rightarrow a)\ \vee$
$(\exists c \in \Def(\Sb) \ldot  c \Rightarrow a) \vee
\blue{a\supportplus a}\}$;\\
$\bullet$ $\Acc(\Sb) \= \{ a \in A | (\forall b\! \in\! A \ldot  b\rightarrow a\ \imply\ b\! \in\! \Def(\Sb)) \wedge (\forall c\in A \ldot  c\Rightarrow a\ \imply\  c\! \in\! \Acc(\Sb) \}$.
\end{definition}

As for the acyclic case,  the semantics of general AFD is dual w.r.t. that of general AFN.  Thus,  we can transform any AFD into an equivalent AFN by reversing supports.

The main difference between the definitions of defeated and acceptable sets for acyclic and general BAF (highlighted in blue) consists in the fact that the new definition of acceptable set $\Acc$ explicitly excludes self-supported arguments (i.e., arguments in a support cycle or supported by an argument in a support cycle), that have been assumed to be defeated.
This means that self-supported arguments are always defeated, independently from the specific set $\Sb$.  Thus, to compute the BAF semantics, it is possible to first state that all arguments belonging to some support cycle are defeated (i.e. they are false arguments), and then, after removing them from the framework, we can compute the semantics of a support-acyclic BAF.  
Intuitively, this is in line with the self-supportless principle in logic programs that ensures no literal in an answer set is supported exclusively by itself or through a cycle of dependencies that lead back to itself.   
As an example, consider the BAF $\<\{a\}, \emptyset, \{ (a,a)\}\>$. 
According to our semantics,  argument $a$ is always defeated.
Any alternative semantics prescribing $a$ as true would be in contrast with the logic formulation of the AF,  
that according to our semantics, is rewritten as a rule $a \leftarrow a$ (literally, $a$ if $a$),  whose unique minimal model is $\emptyset$ (where $a$ is false). 
Notably, as discussed in more detail in Section~\ref{sec:conlusion},
this intuition is also reflected in the stable semantics of Abstract Dialectical Framework~\cite{Brewka-Strass-Ellmauthaler-Wallner-Woltran13-sh}. 
We now give another example to illustrate our semantics.

\begin{figure}[t!]
\centering
\includegraphics[scale=0.55]{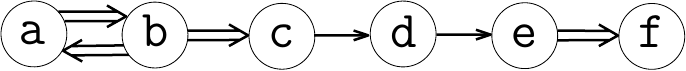}
\caption{\small AFN $\Delta$ of Example~\ref{ex:baf-new-sem}.}\label{fig:AFN}
\end{figure}

\begin{example}\label{ex:baf-new-sem}\rm
Let $\Delta=\tt \< \{\tt a,b,c,d,e,f\},$ 
$\tt \{\tt (c,d), (d,e) \},  \tt \{ (\a,\b),$ $\tt (\b,a),$ $\tt (\b,\c),(\tt e,f)\ \} \>$ be  the  AFN shown on Figure~\ref{fig:AFN}. 
Then, $\Def(\{\tt d\}) $ $=$ $ \{ \tt a,b,c, e,f\}$ and $\Acc(\{\tt d\}) =$ $ \{\tt d\}$. 
The AFN $\Delta$ has a unique complete extension $E = \{\tt d\}$ that is grounded,  preferred and stable.~\hfill~$\square$
\end{example}

As stated next, the semantics introduced extend those defined for acyclic BAF.

\begin{proposition}\label{prop:BAF-generalize}
For any acyclic AFN $\Delta=\<A, \Omega, \Gamma\>$ and semantics $\sigma \in\tt \{ \co, \gr, \pr, \sta \}$,  $\sigma(\<A, \Omega, \Gamma\>)$ computed by means of Definition~\ref{def:BAF-Sem} coincides with $\sigma(\Delta)$ computed by means of Definition~\ref{def:AFN}. 
\end{proposition}
\vspace*{-4mm}
\begin{proof}
The result follows by observing that,  for any set $\Sb \subseteq A $, we have that ${\Def}(\Sb) ={def}(\Sb) $ and ${\Acc}(\Sb) ={acc}(\Sb) $ as there is no argument $a\in A$ s.t.  $a\supportplus a$. 
\end{proof}

As AFD semantics is dual to the on of AFN, an analogous result holds for acyclic AFD.

\subsection{LP mapping}

As done in \cite{Alfano-TPLP} for acyclic BAF, the semantics here presented can be defined even in terms of logic programs under partial stable model semantics.

To this end, we now provide the mappings from general BAF to propositional programs, so that the set of $\sigma$-extensions of any general BAF $\Delta$ is equivalent to that of partial stable models of the corresponding logic program $P_\Delta$. 
The logic rules of $P_\Delta$ are derived from the topology of the AF. 
Basically, the rules in $P_\Delta$ extend the ones defined for AF (Equation $(1)$) as the body of a rule defining an argument $a$ also contains the (positive) conjunction of arguments supporting it.

\begin{definition}\label{def:BAFn-program}
Given an AFN (resp. AFD) $\Delta = \<A, \Omega, \Gamma\>$, then $P_\Delta$ (the propositional program derived from $\Delta$) contains, for each  argument $a \in A$, a rule 
\vspace*{-1mm}
\begin{center}
$
a \leftarrow \bigwedge_{(b,a) \in \Omega} \neg b\ \wedge\  
\bigwedge_{\blue{(c,a)} \in \Gamma} c \hspace*{14mm}
\left(\mbox{resp. } a \leftarrow \bigwedge_{(b,a) \in \Omega} \neg b\ \wedge\  
\bigwedge_{\blue{(a,c)} \in \Gamma} c \right). \ \ \ \ \ \ (2)
$
\end{center}
\end{definition}

Again, we have highlighted in blue the difference between the two mappings.
The next theorem states the equivalence between (general) BAF $\Delta$ and logic program $P_\Delta$ under partial stable model semantics.
As already stated, whenever we use the term BAF we intend either AFN or AFD.

\begin{thm}\label{BAFn:semantics-equivalence}
For any BAF $\Delta$, \ \ $\widehat{\co}(\Delta) = {\cal PS}(P_{\Delta} )$.
\end{thm}
\vspace*{-4mm}
\begin{proof} 
We denote with $\ssup(\Delta)=\{a\!\mid\! a\supportplus a\} \cup \{ b\!\mid\! \exists a\in \ssup(\Delta)\!\ldot\ a\Rightarrow b\}$ when $\Delta$ is an AFN;	 $\ssup(\Delta)=\{a\!\mid\! a\supportplus a\} \cup \{ b\!\mid\! \exists a\in \ssup(\Delta)\!\ldot\ b\Rightarrow a\}$ when $\Delta$ is an AFD.	\\
\noindent
$\bullet\ \widehat{{\co}}(\Delta) \subseteq {\cal PS}(P_{\Delta}).$
We prove that for any $\widehat{\Sb} \in \widehat{{\co}}(\Delta)$, $\widehat{\Sb} \in {\cal PS}(P_{\Delta})$.
Indeed, $P_\Delta$ contains, for each atom $a \in A$, a rule 
$a \leftarrow \bigwedge_{(b,a) \in \Omega} \neg b \wedge \bigwedge_{(c,a) \in \Gamma} c$ (or $a \leftarrow \bigwedge_{(b,a) \in \Omega} \neg b \wedge \bigwedge_{(a,c) \in \Gamma} c$ if $\Delta$ is an AFD).
Moreover, $P_\Delta^{\widehat{\Sb}}$ (the positive instantiation of $P_\Delta$ w.r.t. $\widehat{\Sb}$) contains positive rules defining exactly the arguments in $\Sb$, whose bodies contains only (positive) arguments in $\Sb$.
Since the arguments $a\in A\cap\ssup(\Delta)$ does not appear in $\Sb$ by construction (they are false in the rules' body as appearing in $\Def(\Sb)$), we have that $P_\Delta^{\widehat{\Sb}}$ does not contain cycles of positive literals, and thus 
$T_{P_\Delta^{\widehat{\Sb}}}^\omega(\emptyset) = \Sb$ and $\Sb$ is a PSM for $P_\Delta^{\widehat{\Sb}}$.\\
\noindent
$\bullet\ {\cal PS}(P_{\Delta}) \subseteq \widehat{{\tt co}}(\Delta).$
Consider a PSM $M \in {\cal PS}(P_\Delta)$, $pos(M) = T^\omega_{P_\Delta^M}(\emptyset)$. 
$pos(M) \subseteq A$ is conflict free w.r.t. $\Delta$. Indeed, assuming that there are two arguments $a,b \in pos(M)$ such that $(a,b) \in \Omega$, this means that the rule defining $b$ in $P_\Delta$ contains in the body a literal $\neg a$. This is not possible as in such a case $b \not\in T^\omega_{P_\Delta^M}(\emptyset)$.
Assuming that $\Delta$ is an AFN, and that  $a$ attacks $b$ indirectly through a supported attack $a \rightarrow a_1 \Rightarrow \cdots \Rightarrow a_n \Rightarrow b$. In such a case $a_1,...,a_n,b \not\in T^\omega_{P_\Delta^M}(\emptyset)$.
Assuming that $\Delta$ is an AFD, and that  $a$ attacks $b$ indirectly through a mediated attack $a \rightarrow a_1 \Leftarrow \cdots \Leftarrow a_n \Leftarrow b$. In such a case $a_1,...,a_n,b \not\in T^\omega_{P_\Delta^M}(\emptyset)$. Thus, $pos(M)$ is conflict free.
$pos(M) \subseteq A$ does not contain any argument $a\in \ssup(\Delta)$. Indeed, assuming for contradiction that such an argument $a\in pos(M)$ exists, then there must exists at least one other argument literal $b$ in the body of the rule defining $a$ s.t. (i) $b\in pos(M)$ and $b\in \ssup(\Delta)$. The same holds for any argument of the form of $b$. Thus $M$ is not a minimal model, contradicting the assumption.
Moreover, from Definition \ref{def:BAF-Sem}, considering that $pos(M) = T^\omega_{P_\Delta^M}(\emptyset)$, we derive that $pos(M) = \Acc(pos(M))$. 
\end{proof}

The previous theorem states that the set of complete extensions of any BAF $\Delta$ coincides with the set of PSMs of the derived logic program $P_{\Delta}$.  Consequently,  the set of stable and preferred extensions (resp., the grounded extension) coincide with the set of \textit{total stable} and \textit{maximal-stable} models (resp., the \textit{well-founded} model) of $P_\Delta$~\cite{Sacca97,GelderRS91} .

\begin{example}\label{ex:baf-new-sem-lp}\rm
The propositional program $P_\Delta$ derived from the AFN $\Delta$ of Example~\ref{ex:baf-new-sem} is
$\{ \tt a\leftarrow b;$\ 
$\tt b\leftarrow a;$\ 
$\tt c\leftarrow b;$ \
$\tt e\leftarrow \neg d;$\ 
$\tt d\leftarrow \neg c;$ \
$\tt f\leftarrow e;\}$,
where ${\cal PS}(P_{\Delta})=$ $\widehat{\co}(\Delta)=$ $\{\{\tt \neg a, \neg  b, \neg c, d, \neg e, \neg f\}\}$.~\hfill~$\square$
\end{example}

It is worth noting that, from the above result and the fact that any LP can be converted into an `equivalent' AF~\cite{CaminadaSAD15},  it is also possible to convert any BAF into an `equivalent'  AF in polynomial time.

\subsection{Computational Complexity}

The complexity of  verification and acceptance problems for acyclic BAF are well-know and coincide with the corresponding ones on AF, as any acyclic BAF can be rewritten into an AF. 
In this section we show that the same holds also for general BAF, that is the complexities of the verification and acceptance problems for general BAF are the same of those known for AF.

\begin{proposition}\label{proposition:BAFver}
For any BAF $\Delta = \< A,\Omega,\Gamma\>$, and semantics 
$\sigma  \in\tt \{ \co, \gr, \pr, \sta \}$,  checking whether a set of arguments 
$\Set{S} \subseteq A$ is a $\sigma$-extension for $\Delta$ is: $(i)$ in PTIME for $\sigma \in \{\gr, \co,\sta\}$; and $(ii)$ $coN\!P$-complete for $\sigma =\pr$.
\end{proposition}
\vspace*{-4mm}
\begin{proof}
Lower bounds derive from the complexity results of the same problems for AF~\cite{DvorakD17}, as BAF is a generalization of AF.
As for upper bounds, the proof can be carried out by writing $P_\Delta$ (in PTIME) and check that $(P_\Delta,\widehat{\Sb}=\Sb\cup\{\neg x\mid x\in \Def(\Sb)\},\sigma^*)$ is a true instance of the verification   problem in LP~\cite{Sacca97},  where $\sigma^*=\tt {\cal WF}$ (resp., $\cal PS, TS, MS$) iff $\sigma=\tt gr$ (resp., $\tt co, st, pr$).
As for Theorem~\ref{BAFn:semantics-equivalence} we have that $\widehat{\co}(\Delta)={\cal PS}(P_\Delta)$, the result follows.
\end{proof}

\begin{proposition}
For any BAF $\Delta = \< A,\Omega,\Gamma\>$  and semantics $\sigma  \in\tt \{ \co, \gr, \pr, \sta \}$,  checking whether an argument $g \in A$ is

$\bullet$\ \  credulously accepted under $\sigma$ is:
$(i)$ \ \ in PTIME for $\sigma=\tt gr$; and

\hspace*{49mm}$(ii)$ \,
 $N\!P$-complete for $\sigma \in \{ \co, \sta, \pr \}$; 
 
$\bullet$\ \  skeptically accepted under $\sigma$ is:
$(i)$ \ \ \ in PTIME for $\sigma=\tt gr$;  

\hspace*{49mm}$(ii)$ \ 
$coN\!P$-complete for $\sigma \in\tt \{ co, st \}$; and 

\hspace*{49mm}$(iii)$ 
$\Pi_2^P$-complete for $\sigma =\pr$.
\end{proposition}
\vspace*{-4mm}
\begin{proof}
Same strategy used in the proof of Proposition~\ref{proposition:BAFver} can be used,  where we check that $(P_\Delta,g,\sigma^*)$ is a true instance of the credulous/skeptical acceptance problem in LP~\cite{Sacca97}.
\end{proof}

\section{A new semantics for Recursive BAF}\label{sec:RBAF}

In this section we present new semantics for Recursive BAF (Rec-BAF) frameworks. 
Analogously to the case of BAF, we assume that self-supported arguments (\wrt a set \Sb), that is arguments {$a$} s.t. there exists a cycle ${a} \supportone a_1 \supporttwo \cdots \supportn {a}\ . \{\beta_1,...,\beta_n \} \subseteq \Sb$, are always defeated.
Thus, the definition of defeated elements can be accomplished by adding such a condition.

\begin{definition}\label{def:RAFN-general}
For any general RAFN $\< A, R, T, {\bf s}, {\bf t} \>$ and set $\Sb \subseteq A \cup R \cup T$, we have that: 

\vspace*{1mm}
\noindent
$\bullet\ {\Def}(\Sb) = \{ X \in A \cup R \cup T\ |\
(\exists \alpha \in R \cap \Sb \ldot  \tb(\alpha)=X \wedge {\sb(\alpha) \in \Sb}) \ \vee \\ 
\hspace*{47mm} (\exists~\beta~\in~T \cap~\Sb \ldot \tb(\beta)=X \wedge {\sb~(\beta) \in \Def(\Sb)})\vee \\ 
\hspace*{47mm} \blue{(\exists\ \text{cycle}\ X \supportone a_1 \supporttwo \cdots \supportn X\ . \{\beta_1,...,\beta_n) \} \subseteq \Sb} \}$;

\vspace*{1mm}
\noindent
$\bullet\ \Acc(\Sb)\! =\!\{ X \in A\cup R\cup T\ |\
(\forall \alpha\! \in\! R  \ldot \tb(\alpha)=X\ \imply\ (\alpha \in \Def(\Sb) \vee {{\bf s}(\alpha) \in\! \Def(\Sb)}))\wedge \\
\hspace*{45mm} (\forall \beta\! \in\! T \ldot \tb(\beta)=X\!\ \imply\     (\beta \in \Def(\Sb) \vee {\bf s}(\beta) \in \Acc(\Sb)))  \}$.
\end{definition}

Again, we have highlighted in blue the differences between the definition of defeated and accepted arguments for general and acyclic RAFN.  As for the case of BAF we assume as defeated those arguments that `depends' on a cycle of support $\beta_1,\dots,\beta_n$ and all supports $\beta_1,\dots,\beta_n$ are part of the candidate extension $\Sb$.

\begin{figure}[t!]\centering
\centering
\includegraphics[scale=0.65]{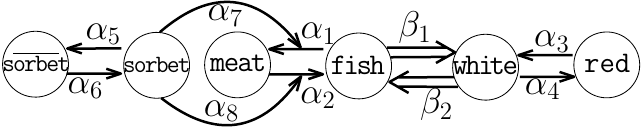}
\caption{\small Rec-BAF of Example~\ref{ex:rafn-new}.}
\label{fig:intro2}
\end{figure}

\begin{example}\label{ex:rafn-new}\rm
Consider the RAFN $\Delta$ (shown in Figure~\ref{fig:intro2}) obtained from that of Figure~\ref{fig:intro} by adding the support $\beta_2$ with $\sb(\beta_2)={\tt white}$ and $\tb(\beta_2)={\tt fish}$, and assume arguments are denoted by their initials. 
The preferred (and also stable) extensions 
under RAFN semantics are:  
$E_1=\{\tt m,r,\bar{s},\alpha_1,\dots,\alpha_8,$ $\beta_1,$ $\beta_2\}
$ and $E_2=\{\tt m,r,s,\alpha_3,\dots,\alpha_8,$ $\beta_1,$ $\beta_2\}$.
Observe that, for $E_3=\{\tt f,w,\bar{s},\alpha_1,\dots,\alpha_8,$ $\beta_1,$ $\beta_2\}$ (resp.,  $E_4=\{\tt f,m,w,s,\alpha_3,\dots,\alpha_8,\beta_1,\beta_2\}$) it holds that $\tt f,w\in$ $\Def(E_3)$  (resp., $\tt f,w\in $ $\Def(E_4)$) and thus   $E_3$ and $E_4$ are not stable (and preferred) extensions.~\hfill~$\square$
\end{example}

\begin{definition}\label{def:ASAF-general}
For any general ASAF $\< A, R, T, {\bf s}, {\bf t} \>$ and set $\Sb \subseteq A \cup R \cup T$, we have that: 

\vspace*{1mm}
\noindent
$\bullet\ {\Def}(\Sb) = \{ X \in A \cup R \cup T\ |\ {(X \in R \wedge \sb(X) \in \Def(\Sb))\  \vee}\\
\hspace*{46.5mm} {(\exists \alpha \in R \cap \Sb \ldot 
	\tb(\alpha)=X)}\ \vee \\  
\hspace*{46.5mm} (\exists \beta \in T \cap \Sb \ldot 
	\tb(\beta)=X\ \wedge\ \sb(\beta) \in \Def(\Sb))\vee \\ 
\hspace*{46.5mm} \blue{(\exists\ \text{cycle}\ X \supportone a_1 \supporttwo \cdots \supportn X\ . \{\beta_1,...,\beta_n \} \subseteq \Sb}) \}$;

\vspace*{1mm}
\noindent
$\bullet\ \Acc(\Sb)\! =\!\{\! X \!\in\! A\cup R\cup T\ |\ {(X\in R\ \imply\ \sb(X) \in \Acc(\Sb))}\wedge\\
\hspace*{44.5mm}\!(\forall \alpha \in R\!\ldot{\bf t}(\alpha)=X\ \imply\ \alpha\in\Def(\Sb)) \wedge \\
\hspace*{44mm} (\forall \beta\in T \ldot  \tb(\beta)=X\ \imply\ (\beta\in \Def(\Sb) \vee \sb(\beta)\in \Acc(\Sb)))\}$.
\end{definition}

Again, we have highlighted in blue the differences between the definition of defeated and accepted arguments for general and acyclic ASAF.  
As for the case of BAF we assume as defeated those arguments that `depends' on a cycle of support $\beta_1,\dots,\beta_n$ and all supports $\beta_1,\dots,\beta_n$ are part of the candidate extension $\Sb$.

\begin{example}\label{ex:asaf-new}\rm
Consider the ASAF $\Delta$ shown in Figure~\ref{fig:intro2}, and assume arguments are denoted by their initials. 
The preferred (and also stable) extensions 
under ASAF semantics are:  
$E_1=\{\tt m,r,$ $\bar{s}, $ $\alpha_2,$ $\alpha_3,$ $\alpha_6,$ $\beta_1,$ $\beta_2\}
$ and $E_2=\{\tt m,r,s,\alpha_3,\alpha_5,\alpha_7,\alpha_8,\beta_1,\beta_2\}$.
Observe that, for $E_3=\{\tt f,w,\bar{s},\alpha_1,\alpha_4,\alpha_6,\beta_1,\beta_2\}$ (resp.,  $E_4=\{\tt f,m,w,s,\alpha_4,\alpha_5,\alpha_7,\alpha_8,\beta_1,\beta_2\}$) it holds that $\tt f,w\in$ $\Def(E_3)$  (resp., $\tt f,w\in $ $\Def(E_4)$) and thus   $E_3$ and $E_4$ are not stable (and preferred) extensions.~\hfill~$\square$
\end{example}

As for BAF our Rec-BAF semantics coincide with that of acyclic Rec-BAF defined in Section~\ref{sec:Preliminaries} whenever the Rec-BAF is support-acyclic, as stated in the following proposition.

\begin{proposition}\label{prop:BAF-generalize}
Let $\Delta$ be a RAFN (resp. ASAF) and $\sigma \in\tt \{ \co, \gr, \pr, \sta \}$ a semantics.   
If $\Delta$ is acyclic,  then $\sigma(\Delta)$ computed by means of Definition~\ref{def:RAFN-sem} (resp.  Definition~\ref{def:ASAF-sem}) coincides with $\sigma(\Delta)$ computed by means of Definition~\ref{def:RAFN-general} (resp.  Definition~\ref{def:ASAF-general}). 
\end{proposition}

\begin{proof}
The result follows by observing that,  for any acyclic Rec-BAF $\Delta=\< A, R, T, {\bf s}, {\bf t} \>$ and set $\Sb \subseteq A \cup R \cup T$, we have that ${\Def}(\Sb) ={def}(\Sb) $ and ${\Acc}(\Sb) ={acc}(\Sb) $. 
\end{proof}

The semantics of general Rec-BAFs  under deductive support (i.e.,  RAFD, and AFRAD) are obtained from the definition of defeated and acceptable elements in acyclic Rec-BAF (i.e., $\nrDef$, $\nrAcc$) by (i) replacing $\nrDef$ and $\nrAcc$ with $\Def$ and $\Acc$ and (ii) adding the condition $\exists\ \text{cycle}\ X \supportone a_1 \supporttwo \cdots \supportn X\ . \{\beta_1,...,\beta_n \} \subseteq \Sb$ in the disjunct condition of $\Def$, as done in Definition~\ref{def:RAFN-general} for RAFN and Definition~\ref{def:ASAF-general} for ASAF.

\subsection{LP mapping}

We now provide the mappings from general Rec-BAF to ground logic programs, so that the set of $\sigma$-extensions of any general Rec-BAF $\Delta$ is equivalent to that of partial stable models of the corresponding logic program $P_\Delta$. 
Basically, the logic rules of the program $P_\Delta$ are derived from the topology of $\Delta$.
The main difference \wrt the rules generated for BAF is that we now have rules defining all elements (arguments, attacks and supports) denoted by $X$.
Thus, for instance, regarding RAFN, a target element $X$ is true if $(i)$ for every attack $\alpha$, either $\alpha$ or the source of $\alpha$ is false, and $(ii)$ for every support $\beta$, either $\beta$ is false or the source of $\beta$ is true.

\begin{definition}\label{def:AFRAS-program-new-sem}
Given an RAFN (resp. RAFD) 
$\Delta = \<A, R, T, {\bf s}, {\bf t}\>$, then 
$P_\Delta $ (the propositional program derived from $\Delta$) contains, for each element $X \in A \cup R \cup T$, a rule  
{\small
\[
X \leftarrow 
	\!\!\!\!\!\!\! \bigwedge_{\alpha \in R \wedge {\bf t}(\alpha)=X} \!\!\!\!\!\!\! (\neg \alpha \vee \neg \sb(\alpha))\wedge \!\!\!\!\!\!\!\! 
	 \bigwedge_{\beta \in T \wedge \tb(\beta)=X} \!\!\!\!\!\!(\neg \beta \vee \sb(\beta)). 
\left(\mbox{resp. }
	 X \leftarrow \!\!\!\!\!\!\!
	\bigwedge_{\alpha \in R \wedge {\bf t}(\alpha)=X} \!\!\!\!\!\!\!(\neg \alpha \vee \neg \sb(\alpha))\wedge \!\!\!\!\!\!
	 \bigwedge_{\beta \in T \wedge \sb(\beta)=X} \!\!\!\!\!\!\!(\neg \beta \vee \tb(\beta)). 
\right)
\]
} 
\end{definition}

\begin{definition}\label{def:AFRAS-program}
For any ASAF (resp. AFRAD) $\Delta = \<A, R, T, {\bf s}, {\bf t}\>$, $P_\Delta $ (the propositional program derived from $\Delta$) contains, for each $X \in A \cup R \cup T$, a rule of the form
{\small
\[
X \leftarrow \varphi(X) \!\
\wedge \!\!\!\!\!\!\!
	\bigwedge_{\alpha \in R \wedge {\bf t}(\alpha)=X}\!\!\!\!\!\!\!\!\!\! \neg \alpha\ \ \  \wedge
	\!\!\!\!\!\!\! \bigwedge_{\beta \in T \wedge \tb(\beta)=X} \!\!\!\!\!\!\!\!\!\!\!\!\!\ (\neg \beta \vee \sb(\beta)) 
\left(\mbox{resp. }
X \leftarrow 
	\varphi(X) \wedge \!\!\!\!\!\!\!
	\bigwedge_{\alpha \in R \wedge {\bf t}(\alpha)=X} \!\!\!\!\!\!\!\!\!\!\!\!\!\ \neg \alpha\ \ \ \ \   \wedge
	 \!\!\!\!\!\!\! \bigwedge_{\beta \in T \wedge \sb(\beta)=X} \!\!\!\!\!\!\!\!\!\!\!\!\!\ (\neg \beta \vee \tb(\beta))
\right)
\]
}
where: $\varphi(X)={\bf s}(X)$ if $\ X\in R$; otherwise, $\varphi(X)=$\true.
\end{definition}

The set of complete extensions of any Rec-BAF $\Delta$ coincides with the set of PSMs of the derived propositional program $P_\Delta$.

\begin{thm}\label{thm:RECBAF-equivLP}
	For any Rec-BAF $\Delta$, \ \  	
$\widehat{\co}(\Delta) = {\cal PS}(P_{\Delta} )$.
\end{thm}
\vspace*{-4mm}
\begin{proof}
The proof is similar to that of Theorem~\ref{BAFn:semantics-equivalence}. The only difference is that complete extensions also contain attacks and supports, whereas the logic program also contains rules defining attacks and supports and, consequently, the partial stable models contain arguments, attacks and supports.$\!\!\!\!\!$
\end{proof}

\begin{example}
{Consider the RAFN $\Delta$ shown in Figure~\ref{fig:intro2}, and assume arguments are denoted by their initials. The propositional program $P_\Delta$ derived from $\Delta$ is as follows: 
\begin{center}
\begin{tabular}{cccc}
$\tt m\leftarrow \neg \alpha_1 \vee f;$ & $\tt f\leftarrow (\neg \alpha_2\vee m)\wedge(\neg \beta_2\vee w);$ & $\tt s\leftarrow \neg \alpha_6\vee \bar{s};$ & $\tt \bar{s}\leftarrow \neg \alpha_5\vee s;$ \\
 $\tt r\leftarrow \neg \alpha_4\vee w;$ &  $\tt w\leftarrow (\neg \alpha_3\vee r)\wedge(\neg \beta_1\vee f);$ &$\tt \alpha_1\leftarrow \neg \alpha_7\vee s;$ & $\tt \alpha_2\leftarrow \neg \alpha_8\vee s;$\\ 
$\tt \alpha_3\leftarrow ;\  \alpha_4\leftarrow ;$ & $\tt \alpha_5\leftarrow;	 \  \alpha_6\leftarrow ;$ & $\tt \alpha_7\leftarrow ;\  \alpha_8\leftarrow ;$ &$\tt \beta_1\leftarrow ;\ \beta_2\leftarrow ;$
\end{tabular}
\end{center}
\noindent 
where ${\cal PS}(P_{\Delta})=$ $\widehat{\co}(\Delta)=\{E_0,E_1,E_2,E_3,E_4,E_5\}$,  with:
$E_0=\{\tt \alpha_3,\dots,\alpha_8, \beta_1,\beta_2\}$,
$E_1=E_0\cup\{\tt m, r\}$, 
$E_2=E_0\cup\{\tt m, s\}$, 
$E_3=E_0\cup\{\tt m, r, s\}$, 
$E_4=E_0\cup\{\tt \bar{s}, \alpha_1,\alpha_2\}$,  and 
$E_5=E_0\cup\{\tt m, r, \bar{s}, \alpha_1,\alpha_2\}$.~\hfill~$\Box$
}
\end{example}

Observe that, from the above result and the fact that any LP can be converted into an `equivalent' AF~\cite{CaminadaSAD15},  it is also possible to convert any Rec-BAF into an `equivalent'  AF in polynomial time.

\subsection{Computational Complexity}

The complexity of verification and acceptance problems for acyclic Rec-BAF are well-known and coincide with the corresponding ones on AF, as any acyclic Rec-BAF can be rewritten into an AF. 
In this section we show that the same holds also for general Rec-BAF, that is the complexities of the verification and acceptance problems for general BAF are the same of those known for AF, as stated in the following two propositions.

\begin{proposition}\label{propositionrecbaf-ver}
For any Rec-BAF $\< A, R, T, {\bf s}, {\bf t} \>$ and semantics 
$\sigma  \in\tt \{ \co, \gr, \pr, \sta \}$, 
checking whether a set of arguments 
$\Set{S} \subseteq A$ is a $\sigma$-extension for $\Delta$ is $(i)$ in PTIME for $\sigma \in \{\gr, \co,\sta\}$; and $(ii)$ $coN\!P$-complete for $\sigma =\pr$.
\end{proposition}
\vspace*{-2mm}

\begin{proof}Lower bounds derive from the complexity results of the same problems for AF~\cite{DvorakD17}, as Rec-BAF is a generalization of AF.
Regarding upper bounds, the proof can carried out by writing $P_\Delta$ (in PTIME) and check that $(P_\Delta,\widehat{\Sb}=\Sb\cup\{\neg x\mid x\in \Def(\Sb)\},\sigma^*)$  is a true instance of the verification problem in LP~\cite{Sacca97}, where $\sigma^*=\tt {\cal WF}$ (resp., $\cal PS, TS, MS$) iff $\sigma=\tt gr$ (resp., $\tt co, st, pr$).
As for Theorem~\ref{thm:RECBAF-equivLP} we have that $\widehat{\co}(\Delta)={\cal PS}(P_\Delta)$, the result follows.
\end{proof}

\begin{proposition}
For any Rec-BAF $\< A, R, T, {\bf s}, {\bf t} \>$ and semantics 
$\sigma  \in\tt \{ \co, \gr, \pr, \sta \}$, 
checking whether an argument $g \in A$ is

$\bullet $ credulously accepted under $\sigma$ is:
$(i)$
in PTIME for $\sigma=\tt gr$; and

\hspace*{49mm}$(ii)$
 $N\!P$-complete for $\sigma \in \{ \co, \sta, \pr \}$; 
 
$\bullet$ skeptically accepted under $\sigma$ is:
$(i)$ in PTIME for $\sigma=\tt gr$;

\hspace*{48mm}$(ii)$
$coN\!P$-complete for $\sigma \in\tt \{ co, st \}$; and 

\hspace*{48mm}$(iii)$ 
$\Pi_2^P$-complete for $\sigma =\pr$.
\end{proposition}
\vspace*{-2mm}

\begin{proof}
Same strategy used in the proof of Proposition~\ref{propositionrecbaf-ver} can be used,  where we check that $(P_\Delta,g,\sigma^*)$ is a true instance of the credulous/skeptical acceptance problem in LP~\cite{Sacca97}.
\end{proof}

It turns out that general Rec-BAF are as expressive as BAF and AF,  though more general relations (i.e.,  cyclic supports and recursive relations) can be easily expressed.

\section{Related Work}
Bipolarity in argumentation is discussed in~\cite{AmgoudCL04}, where a formal definition of bipolar argumentation framework (BAF) extending Dung's AF by including supports is provided.
A survey of different approaches to support in argumentation can be found in~\cite{CohenGGS14},  and a survey on different recursive AF-based frameworks can be found in~\cite{CayrolCLS21}.
However, a semantics for AFNs with cyclic supports have been recently defined in the literature~\cite{NouiouaB23}.  
Essentially, it is based on avoiding considering the contribution of \textit{incoherent} (set of) arguments, that are those occurring in support-cycles or (transitively) supported by arguments in support-cycles.
In the same spirit,  our approach ensures that incoherent arguments are always defeated (i.e., appearing in $\Def$).  
A semantics for general RAFNs has been recently defined in~\cite{LagasquieSchiex23}.
We believe their proposal is quite intricate as it relies on numerous definitions, whereas ours seamlessly extends the definitions of defeated and acceptable elements, previously defined for AF(N) in an elegant and uniform way.

There has been an increasing interest in studying the relationships between argumentation frameworks and logic programming (LP)~\cite{CaminadaSAD15,Alfano-TPLP}.
In particular,  the semantic equivalence between complete extensions in AF and 3-valued stable models in LP was first established in~\cite{WuCG09}.
Then, the relationships of LP with AF have been further studied in~\cite{CaminadaSAD15}.
A one-to-one correspondence between extensions of general AFN and corresponding normal logic program (LP) has been proposed in~\cite{NouiouaB23},  under complete-based semantics.  
A logical encoding able to characterize the semantics of general RAFN has been proposed in~\cite{LagasquieSchiex23}.  
However, the rules defining acceptability of arguments are a bit involved,  as they require several predicate symbols.   
Nevertheless, in \cite{NouiouaB23,LagasquieSchiex23} no attention has been devoted to deductive supports.

Efficient mappings from AF to \emph{Answer Set Programming} (i.e. LP with \emph{Stable Model} semantics \cite{GelfondL88}) have been investigated as well \cite{SakamaR17,GagglMRWW15}.
The well-know AF system ASPARTIX~\cite{DvorakGRWW20} is implemented by rewriting the input AF into an ASP program and using an ASP solver to compute extensions. 

Our work is complementary to approaches providing the semantics for an AF-based framework by using meta-argumentation, that is, by relying on a translation from a given AF-based framework to an AF~\cite{CohenGGS15,AlfanoGP18c,AlfanoGP18}.
In this regard, we observe that meta-argumentation approaches have the drawback of making it a bit  difficult to understand the original meaning of arguments and interactions, once translated into the resulting meta-AF.
In fact, those approaches rely on translations that generally require adding several meta-arguments and meta-attacks to the resulting meta-AF in order to model the original interactions.
Concerning approaches that provide the semantics of argumentation frameworks by LPs~\cite{CaminadaSAD15}, we observe that a logic program for an AF-based framework can be obtained by first flattening the given framework into a meta-AF and then converting it into a logic program.
The so-obtained program contains the translation of meta-arguments and meta-attacks that make the program much more verbose and difficult to understand (because not straightly derived from the given extended AF framework) in our opinion, compared with the direct translation we proposed.
Moreover, the proposed approach uniformly deals with several AF-based frameworks.
Other extensions of the Dung's framework not explicitly discussed in this paper are also captured by our technique as they are special cases of some of those studied in this paper.
This is the case of \emph{Extended AF (EAF)} and \emph{hierarchical EAF}, which extend AF by allowing second order and stratified attacks, respectively \cite{Modgil-AI-09}, that are special cases of recursive attacks.

Three-valued semantics have been also explored in the \textit{Abstract Dialectical Framework} (ADF) \cite{Brewka-Strass-Ellmauthaler-Wallner-Woltran13-sh,StrassW15,Brewka20,BaumannH23},  that allows to explicit acceptance conditions over arguments in the form of propositional logic formulae.  Recently,  it has been shown that its semantics can be captured by the (monotonic three-valued) possibilistic logic~\cite{HeyninckKRST22}.  In particular, the semantics of an ADF relies on a characteristic operator which takes as an input a three-valued interpretation $\nu$ and returns an interpretation by considering all possible  two-valued completions of $\nu$.  
We point out that (general) Rec-BAF can be modeled by ADF,  and in particular that verification and acceptance reasoning in Rec-BAF can be reduced to ADF. This is backed by the computational 
complexity of the two frameworks~\cite{StrassW15} as ADF is strictly more expressive than Rec-BAF under complete, preferred and stable semantics (one level higher in the polynomial hierarchy).  Moreover,  less expressive subclasses of ADF have been also explored.  
For instance, the subclass called bipolar ADFs (BADFs) has been shown to exhibit complexity comparable to that of AF (and to that of Rec-BAF), as is it possible to avoid considering all the possible two-valued completions through the application of Kleene logic~\cite{BaumannH23}.
Exploring the connection between subclasses of ADF and Rec-BAF is a possible direction for future work.

\section{Conclusions}\label{sec:conlusion}

In this paper we jointly tackled two relevant aspects that were so far considered separately by the community of argumentation: extending Dung's abstract argumentation framework with recursive attacks and (general) supports, and show that the semantics for specific BAF and Rec-BAF frameworks can be defined by very simple and intuitive modifications of that defined for the case of AF.
We presented in an elegant and uniform way the semantics of several (possibly cyclic) AF-based frameworks,  including those for which a semantics has never been proposed. 

Our semantics is inspired by the self-supportless principle in logic programs,  that ensures no literal in an answer set is supported exclusively by itself or through a cycle of dependencies that lead back to itself.   This principle is essential for maintaining the integrity of the answer set, ensuring that it is grounded in the logic rules provided and does not rely on circular reasoning.  
It is worth noting that,  this intuition is also reflected in the stable semantics of 
ADF~\cite{Brewka-Strass-Ellmauthaler-Wallner-Woltran13-sh}. 
Indeed,  the basic intuition is that all elements of a stable model should have a non-cyclic justification.  
For instance,  considering the (bipolar) ADF having two statements $a$ and $b$ where the acceptance condition of $a$ is $b$ and vice versa,  the set $\{a,b\}$ is not a stable model. 
Thus, similarly to answer-set and ADF stable semantics, a principle of our semantics is that arguments cannot be circularly justified.

Finally, we believe that our approach is also complementary to that using intermediate translations to AF to define the Rec-BAF semantics.
Indeed,  our approach can also be used to provide additional tools for computing complete extensions using answer set solvers \cite{ASP-Systems} and program rewriting techniques \cite{unfolding,SakamaR17}.

\bibliographystyle{tlplike}
\bibliography{refs}

\clearpage

\section*{Appendix A - Partial Stable Models}
The semantics of a logic program is given by the set of its partial stable models (PSMs) (corresponding to complete extensions of AFs \cite{CaminadaSAD15}).
We summarize the basic concepts underlying the notion of PSMs \cite{Sacca97}. 

A (normal) logic program (LP) is a set of rules of the form
$A \leftarrow B_1 \wedge \cdots \wedge B_n$, with $n \geq 0$, where $A$ is an atom, called head, and $B_1\wedge \cdots \wedge B_n$ is a conjunction of literals, called body. 
We consider programs without function symbols.
Given a program $P$, $ground(P)$ denotes the set of all ground instances of the rules in $P$. 
The Herbrand Base of a program $P$, i.e., the set of all ground
atoms which can be constructed using predicate and constant symbols occurring in $P$, is denoted by $B_P$, whereas $\neg B_P$ denotes the set $\{ \neg A \mid A \in B_P \}$. 
Analogously, for any set $S \subseteq B_P \cup \neg B_P$, $\neg S$ denotes the set $\{ \neg A \mid A \in S \}$, where $\neg \neg A = A$.
Given $I \subseteq B_P \cup \neg B_P$, $pos(I)$ (resp., $neg(I)$) stands for $I \cap B_P$ (resp., $\neg I \cap B_P$).
$I$ is \emph{consistent} if $pos(I) \cap \neg neg(I) = \emptyset$, otherwise $I$ is \emph{inconsistent}.

Given a program $P$, $I \subseteq B_P \cup \neg B_P$
is an \emph{interpretation}  of $P$ if $I$ is consistent. 
Also, $I$ is \emph{total} if 
$pos(I) \cup neg(I) = B_P$, 
\emph{partial} otherwise.
A partial interpretation $M$ of a program $P$ is a
\emph{partial model} of $P$ if for each $\neg A \in M$ every rule in $ground(P)$ having as head $A$ contains at least one body literal $B$ such that $\neg B \in M$.
Given a  program $P$ and a partial model $M$, the positive instantiation of $P$ w.r.t. $M$, denoted by $P^M$, is obtained from $ground(P)$ by deleting:\
$(a)$
each rule containing a negative literal $\neg A$ such that $A \in pos(M)$;\
$(b)$\
each rule containing a literal $B$ such that neither $B$ nor $\neg B$ is in $M$;\
$(c)$
all the negative literals in the remaining rules.
$M$ is a partial stable model of $P$ iff $M$ is the minimal model of $P^M$.
Alternatively, $P_M$ could be built by replacing every negated body literal in $ground(P)$ by its truth value.

The set of partial stable models of a logic program $P$, denoted by ${\cal PS}(P)$, define a meet semi-lattice. 
The \textit{well-founded} model (denoted by ${\cal WF}(P)$) and the \textit{maximal-stable} models ${\cal MS}(P)$, are defined by considering $\subseteq$-minimal and $\subseteq$-maximal elements. 
The set of (total) \textit{stable} models (denoted by ${\cal TS}(P)$) is obtained by considering the maximal-stable models which are total.

The semantics of a logic program is given by the set of its partial stable  models or by one of the restricted sets above recalled.
\end{document}